\documentclass[11pt]{article}
\usepackage[utf8x]{inputenc}
\usepackage[a4paper, left=3cm, right=3cm, top=3cm, bottom=3cm]{geometry}%

\usepackage{amsmath,amsfonts,bm}

\def\1{\bm{1}}

\def\eps{{\epsilon}}

\DeclareMathAlphabet{\mathsfit}{\encodingdefault}{\sfdefault}{m}{sl}
\SetMathAlphabet{\mathsfit}{bold}{\encodingdefault}{\sfdefault}{bx}{n}

\newcommand{\E}{\mathbb{E}}

\newcommand{\R}{\mathbb{R}}

\newcommand{\softmax}{\mathrm{softmax}}

\newcommand{\softplus}{\zeta}

\newcommand{\Proba}{\mathbb{P}}

\newcommand{\diag}{\mathrm{diag}}
\newcommand{\lse}{\mathrm{LSE}}

\def\norm#1{\left\| #1 \right\|}
\def\abs#1{\left| #1 \right|}
\def\prth#1{\left( #1 \right)}
\def\sqprth#1{\left[ #1 \right]}
\def\braces#1{\left\{ #1 \right\}}
\def\scprod#1#2{\left\langle #1, #2 \right\rangle}

\newcommand{\poly}{\mathrm{poly}}
\newcommand{\loss}{\mathrm{loss}}

\usepackage{natbib}
\usepackage{graphicx}
\usepackage{hyperref}
\usepackage{url}
\usepackage{textcomp}
\usepackage{amsmath,amsfonts,amssymb,amsthm}
\usepackage{mathtools}
\usepackage{breqn}
\usepackage{bbm}
\usepackage{url}
\usepackage{times}
\usepackage{color}

\newtheorem{assumption}{Assumption}
\newtheorem{proposition}{Proposition}
\newtheorem{theorem}{Theorem}
\newtheorem{lemma}{Lemma}

\definecolor{myred}{rgb}{0.77, 0.0, 0.1}
\definecolor{newgreen}{RGB}{0,153,0}
\definecolor{myturq}{rgb}{0.1, 0.7, 0.7}

\usepackage{hyperref}

\hypersetup{
    colorlinks,
    linkcolor=myred, 
    citecolor=blue, 
    linktoc=all 
  }

\title{About contrastive unsupervised representation learning for classification and its convergence}

\author{Ibrahim Merad \and Yiyang Yu \and Emmanuel Bacry \and
	St\'ephane Ga\"iffas
}

\author{
  Ibrahim Merad%
  \thanks{LPSM, UMR 8001, Universit\'e de Paris, Paris, France, \texttt{imerad@lpsm.paris}}\\
  \and
  Yiyang Yu%
  \thanks{LPSM, UMR 8001, Universit\'e de Paris, Paris, France, \texttt{yyu@lpsm.paris}}
  \and 
  Emmanuel Bacry%
  \thanks{CEREMADE, Universit\'e Paris-Dauphine, \texttt{bacry@ceremade.dauphine.fr}}
  \and
  Stéphane Gaïffas%
  \thanks{LPSM, UMR 8001, Universit\'e de Paris, Paris, France, \texttt{gaiffas@lpsm.paris}}
  \thanks{DMA, CNRS UMR 8553, Ecole normale sup\'erieure, Paris, France}
  \and   
}

\begin{document}

\maketitle

\begin{abstract}
Contrastive representation learning has been recently proved to be very efficient for self-supervised training. These methods have been successfully used to train encoders which perform comparably to supervised training on downstream classification tasks. 
A few works have started to build a theoretical framework around contrastive learning in which guarantees for its performance can be proven. We provide extensions of these results to training with multiple negative samples and for multiway classification. 
Furthermore, we provide convergence guarantees for the minimization of the contrastive training error with gradient descent of an overparametrized deep neural encoder, and provide some numerical experiments that complement our theoretical findings.

\end{abstract}

\section{Introduction}
\label{sec:introduction}

The aim of this work is to provide additional theoretical guarantees for \emph{contrastive learning}~\citep{DBLP:journals/corr/abs-1807-03748}, which corresponds to methods allowing to learn useful data representations in an \emph{unsupervised} setting.
Unsupervised representation learning was initially approached with a fair amount of success by training through the minimization of losses coming from ``pretext'' tasks, a technique known as \emph{self-supervision} \citep{DBLP:journals/corr/abs-1708-07860}, where labels can be automatically constructed.
Notable examples of pretext tasks in computer vision include colorization~\citep{DBLP:journals/corr/ZhangIE16}, transformation prediction~\citep{DBLP:journals/corr/abs-1803-07728,DBLP:journals/corr/DosovitskiySRB14} or predicting patch relative positions~\citep{DBLP:journals/corr/DoerschGE15}. 
Some theoretical guarantees~\citep{lee2020predicting} were recently proposed to support training on pretext tasks. 

Contrastive learning is also known to be very effective for pretraining supervised methods~\citep{chen2020simple,chen2020big,grill2020bootstrap,caron2020unsupervised}, where we can observe that, quite surprisingly, the gap between unsupervised and supervised performance has been closed for tasks such as image classification: the use of a pretrained image encoder on top of simple classification layers, that are trained on a fraction of the labels available, allows to achieve an accuracy comparable to that of a fully supervised end-to-end training~\citep{DBLP:journals/corr/abs-1905-09272, grill2020bootstrap}.
Contrastive methods show also strong success in natural language processing~\citep{logeswaran2018efficient, mikolov2013distributed, DBLP:journals/corr/abs-1810-04805, DBLP:journals/corr/abs-1807-03748}, video classification~\citep{sun2019learning}, reinforcement learning~\citep{srinivas2020curl} and time-series~\citep{franceschi2019unsupervised}.

Although the papers cited above introduce methods with considerable variations, they mostly agree on the following basic pretraining approach: 
provided a dataset, an encoder is trained using a contrastive loss whose minimization allows to learn embeddings that are \emph{similar} for pairs of samples (called the \emph{positives}) that are close to each other (such as pairs of random data augmentations of the same image, see~\citet{he2020momentum, chen2020simple}), while such embeddings are \emph{contrasted} for dissimilar pairs (called the \emph{negatives}).
 
However, despite growing efforts~\citep{saunshi2019theoretical, wang2020understanding}, as of today, few theoretical results have been obtained. 
For instance, there is still no clear theoretical explanation of how a supervised task could benefit from an upstream  unsupervised pretraining phase, or of what could be the theoretical guarantees for the convergence of the minimization procedure of the contrastive loss during this pretraining phase. 
Getting some answers to these questions would undoubtedly be a step towards a better theoretical understanding of contrastive representation learning.

Our contributions in this paper are twofold. 
In Section~\ref{sec:sec3}, we provide new theoretical guarantees for the classification performance of contrastively trained models in the case of multiway classification tasks, using \emph{multiple} negative samples. We extend results from~\citet{saunshi2019theoretical} to show that unsupervised training performance reflects on a subsequent classification task in the case of multiple tasks and when a high number of negative samples is used. 
In Section~\ref{subsec:conv}, we prove a convergence result for an \emph{explicit} algorithm (gradient descent), when training overparametrized deep neural network for unsupervised contrastive representation learning. We explain how results from~\citet{allen2019convergence} about training convergence of overparametrized deep neural networks can be applied to a contrastive learning objective.
The results and major assumptions of both Sections~\ref{sec:sec3} and~\ref{subsec:conv} are illustrated in Section~\ref{sec:experiments} through experiments on a few simple datasets. 

\section{Related work}

A growing literature attempts to build a theoretical framework around contrastive learning and to provide justifications for its success beyond intuitive ideas.
In~\citet{saunshi2019theoretical} a formalism is proposed together with results on classification performance based on unsupervisedly learned representation.
However, these results do not explain the performance gain that is observed empirically~\citep{chen2020simple, he2020momentum} when a high number of negative samples are used, while the results proposed in Section~\ref{sec:sec3} below hold for an arbitrary large number of negatives (and decoupled from the number of classification tasks).
A more recent work~\citep{wang2020understanding} emphasizes the two tendencies encouraged by the contrastive loss: 
the encoder's outputs are incentivized to spread evenly on the unit hypersphere, and encodings of same-class samples are driven close to each other while those of different classes are driven apart.
Interestingly, this work also shows how the tradeoff between these two aspects can be controlled, by introducing weight factors in the loss leading to improved performance. 
\citet{chuang2020debiased} considers the same setting as~\citet{saunshi2019theoretical} and addresses the bias problem that comes from collisions between positive and negative sampling in the unsupervised constrastive loss. 
They propose to simulate unbiased negative sampling by assuming, among other things, extra access to positive sampling. 
However, one has to keep in mind that an excessive access to positive sampling gets the setting closer to that of supervised learning. 

In a direction that is closer to the result proposed in Section~\ref{subsec:conv} below,~\cite{wen2020convergence} provides a theoretical guarantee on the training convergence of gradient descent for an overparametrized model that is trained with an unsupervised contrastive loss, using earlier works by~\citet{allen2019convergence}. 
However, two separate encoders are considered instead of a single one: one for the query, which corresponds to a sample from the dataset, and one for the (positive and negative) samples to compare the query to.
In this setting, it is rather unclear how the two resulting encoders are to be used for downstream classification. 
In Section~\ref{subsec:conv} below, we explain how the results from~\citet{allen2019convergence} can be used for the more realistic setting of a single encoder, by introducing a reasonable assumption on the encoder outputs.

\section{Unsupervised training improves supervised performance}
\label{sec:sec3}

In this section, we provide new results in the setting previously considered in~\citet{saunshi2019theoretical}.
We assume that data are distributed according to a finite set $\mathcal{C}$ of latent classes, and denote $N_{\mathcal{C}} = \text{card}(\mathcal{C})$ its cardinality. 
Let $\rho$ be a discrete distribution over $\mathcal{C}$ that is such that
\begin{equation*}
    \sum_{c \in \mathcal{C}} \rho(c) = 1 \quad \text{and} \quad \rho(c) > 0
\end{equation*}
for all $c \in \mathcal C$. We denote $\mathcal{D}_c$ a distribution over the feature space $\mathcal{X}$ from a class $c \in \mathcal C$.
In order to perform unsupervised contrastive training, on the one hand we assume that we can sample \emph{positive} pairs $(x, x^+)$ from the distribution
\begin{equation}
    \label{Dsim} 
    \mathcal{D}_{\mathrm{sim}}(x, x^+) = \sum_{c \in \mathcal{C}} \rho(c) \mathcal{D}_c(x)\mathcal{D}_c(x^+),
\end{equation}
namely, $(x, x^+)$ is sampled as a mixture of independent pairs conditionally to a shared latent class, sampled according to $\rho$.
On the other hand, we assume that we can sample \emph{negative} samples $x^-$ from the distribution
\begin{equation}
    \label{Dneg}
    \mathcal{D}_{\mathrm{neg}}(x^-) = \sum_{c \in \mathcal{C}} \rho(c) \mathcal{D}_c(x^-).
\end{equation}
Given $k \leq N_{\mathcal{C}}-1$, a $(k+1)$-way classification task is a subset $\mathcal{T} \subseteq \mathcal{C}$ of cardinality $\abs{\mathcal{T}} = k+1$, which induces the conditional distribution
\begin{equation*}
  \mathcal{D}_{\mathcal{T}} (c) = \rho(c \mid c \in \mathcal{T})
\end{equation*}
for $c \in \mathcal{C}$ and we define
\begin{equation*}
   \mathcal{D}_{\mathcal{T}}(x,c) = \mathcal{D}_{\mathcal{T}}(c) \mathcal{D}_c(x).  
\end{equation*}
In particular, we denote as $\mathcal{C}$, whenever there is no ambiguity, the $N_{\mathcal{C}}$-way classification task where the labels are sampled from $\rho$, namely $\mathcal{D}_{\mathcal{C}}(x,c) = \rho(c) \mathcal{D}_c(x)$.

\paragraph{Supervised loss and mean classifier.}

For an encoder function $f: \mathcal{X} \to \mathbb{R}^{d}$, we define a supervised loss (cross-entropy with the best possible linear classifier on top of the representation) over task $\mathcal{T}$ as
\begin{equation}
    \label{LsupT}
    L_{\mathrm{sup}}(f, \mathcal{T}) =\inf_{W \in \R^{\abs{\mathcal{T}} \times d}} \E_{(x, c) \sim \mathcal{D_{\mathcal{T}}}} \sqprth{ - \log \prth{\frac{\exp \prth{W f(x)}_c}{\sum_{c' \in \mathcal{T}} \exp \prth{W f(x)}_{c'}}}}.
\end{equation}
Then, it is natural to consider the \emph{mean} or \emph{discriminant} classifier with weights $W^\mu$ which stacks, for $c \in \mathcal{T}$,
the vectors 
\begin{equation}
    \label{Wmu}
    W^{\mu}_{c,:} = \E_{x \sim \mathcal{D}_c} \sqprth{f(x)}
\end{equation}
and whose corresponding (supervised) loss is given by
\begin{equation}
    \label{LmusupT}
    L_{\mathrm{sup}}^\mu(f, \mathcal{T}) = \E_{(x, c) \sim \mathcal{D}_{\mathcal{T}}} \sqprth{ - \log \prth{ \frac{\exp \prth{W^\mu f(x)}_c}{\sum_{c' \in \mathcal{T}} \exp \prth{W^\mu f(x)}_{c'}} }}.
\end{equation}
Note that, obviously, one has $L_{\mathrm{sup}}(f, \mathcal{T}) \leq L_{\mathrm{sup}}^\mu(f, \mathcal{T})$.

\paragraph{Unsupervised contrastive loss.}

We consider the unsupervised contrastive loss \emph{with $N$ negative samples} given by
\begin{equation}
    \label{L_un_k}
    L_{\mathrm{un}}^N(f) = \E_{\substack{(x, x^+) \sim \mathcal{D}_{\mathrm{sim}} \\ X^- \sim \mathcal{D}_{\mathrm{neg}}^{\otimes N}}} \sqprth{ - \log \prth{ \frac{\exp \prth{f(x)^T f(x^+)}}{\exp \prth{f(x)^T f(x^+)} + \sum_{x^- \in X^-} \exp \prth{f(x)^T f(x^-)}}}},
\end{equation}
where $\mathcal{D}_{\mathrm{sim}}$ is given by Equation~\eqref{Dsim} and where $\mathcal{D}_{\mathrm{neg}}^{\otimes N}$ stands for the $N$ tensor product of the $\mathcal{D}_{\mathrm{neg}}$ distribution given by Equation~\eqref{Dneg}.
When a single negative sample is used ($N=1$), we will use the notation $L_{\mathrm{un}}(f) = L_{\mathrm{un}}^1(f)$.
In the rest of the paper, $N$ will stand for the number of negatives used in the unsupervised loss~\eqref{L_un_k}.

\subsection{Inequalities for unsupervised training with multiple classes}

The following Lemma states that the unsupervised objective with a single negative sample can be related to the supervised loss for which the target task is classification over the whole set of latent classes $\mathcal{C}$.
 \begin{lemma}
    \label{lem81}
 For any encoder $f:\mathcal{X} \to \mathbb{R}^d$\textup, one has
    \begin{equation}\label{ineq31}
         L_{\mathrm{sup}}(f, \mathcal{C}) \leq L_{\mathrm{sup}}^\mu (f, \mathcal{C}) \leq \frac{1}{p^\rho_{\mathrm{min}}} L_{\mathrm{un}}(f) + \log N_{\mathcal{C}},
    \end{equation}
    where $p^\rho_{\mathrm{min}} = \min_c \rho(c)$.
\end{lemma}
The proof of Lemma~\ref{lem81} is given in the appendix, and uses a trick from Lemma~4.3 in~\citet{saunshi2019theoretical} relying on Jensen's inequality.
This Lemma relates the unsupervised and the supervised losses, a shortcoming being the introduction of $p^\rho_{\mathrm{min}}$, which is small for a large $N_{\mathcal{C}}$ since obviously $p^\rho_{\mathrm{min}} \leq 1 / N_{\mathcal{C}}$.

The analysis becomes more difficult with a larger number of negative samples.  
Indeed, in this case, one needs to carefully keep track of how many distinct classes will be represented by each draw.
This is handled by Theorem~B.1 of~\citet{saunshi2019theoretical}, but the bound given therein only estimates an expectation of the supervised loss w.r.t. the random subset of classes considered (so called tasks). 
For multiple negative samples, the approach adopted in the proof of  Lemma~\ref{lem81} above further degrades, since $p^\rho_{\mathrm{min}}$ would be replaced by the minimum probability among tuple draws, an even much smaller quantity.

We propose the following Lemma, which assumes that the number of negative samples is large enough compared to the number of latent classes.
\begin{lemma}
    \label{lem32} 
    Consider the unsupervised objective with $N$ negative samples as defined in Equation~\eqref{L_un_k} and assume that $N$ satisfies $N = \Omega(N_{\mathcal{C}} \log N_{\mathcal{C}})$.
    Then\textup, we have
    \begin{equation}
        \label{ineq32}
        L_{\mathrm{sup}}(f, \mathcal{C}) \leq L^\mu_{\mathrm{sup}}(f, \mathcal{C}) \leq \frac{1}{p^\rho_{cc}(N)} L^N_{\mathrm{un}}(f),
    \end{equation}
    where $p^\rho_{cc}(N)$ is the probability to have all coupons after $N$ draws in an $N_{\mathcal{C}}$-coupon collector problem with draws from $\rho$.
\end{lemma}

The proof of Lemma~\ref{lem32} is given in the appendix.
In this result, $p^{\rho}_{cc}(N)$ is related to the following coupon collector problem. 
Assume that $\rho$ is the uniform distribution over $\mathcal C$ and let $T$ be the random number of necessary draws until each $c \in \mathcal C$ is drawn at least once.
It is known~(see for instance~\cite{motwani_raghavan_1995}) that the expectation and variance of $T$ are respectively given by $N_{\mathcal{C}} H_{N_{\mathcal{C}}}$ and $(N_{\mathcal{C}} \pi)^2 / 6$,
where $H_n$ is the $n$-th harmonic number $H_n = \sum_{i=1}^n 1 / i$.
This entails using Chebyshev's inequality that
\begin{equation*}
    \Proba\prth{\abs{T - N_{\mathcal{C}} H_{N_{\mathcal{C}}}} \geq \beta N_{\mathcal{C}}} \leq \frac{\pi^2}{6 \beta^2}
\end{equation*}
for any $\beta >0$, so that
whenever $\rho$ is sufficiently close to a uniform distribution and $N= \Omega(N_{\mathcal{C}} \log N_{\mathcal{C}})$, the probability $p^{\rho}_{cc}$ is reasonably high.
Due to the randomness of the classes sampled during training, it is difficult to obtain a better inequality than Lemma~\ref{lem32} if we want to upper bound $L^N_{\mathrm{un}}(f)$ by the supervised $L_{\mathrm{sup}}(f, \mathcal{C})$ on all classes.
However, the result can be improved by considering the average loss over tasks $L_{\mathrm{sup},k}(f)$, as explained in the next Section.

\subsection{Guarantees on the average supervised loss}

In this Section, we bound the average of the supervised classification loss on tasks that are subsets of $\mathcal{C}$.
Towards this end, we need to assume (only in this Section) that $\rho$ is uniform.
We consider supervised tasks consisting in distinguishing one latent class from $k$ other classes, given that they are distinct and uniformly sampled from $\mathcal{C}$. 
We define the average supervised loss of $f$ for $(k+1)$-way classification as
\begin{equation}
    L_{\mathrm{sup},k}(f) = \E_{\mathcal{T} \sim \mathcal{D}^{k+1}} \sqprth{L_{\mathrm{sup}} \prth{f, \mathcal{T}} },
\end{equation}
where $\mathcal{D}^{k+1}$ is the uniform distribution over $(k+1)$-way tasks, which means uniform sampling of $\{c_1, \cdots, c_{k+1}\}$ \emph{distinct} classes in $\mathcal{C}.$
We define also the average supervised loss of the mean classifier
\begin{equation}
    L_{\mathrm{sup}, k}^\mu(f) = \E_{\mathcal{T} \sim \mathcal{D}^{k+1}} \sqprth{L_{\mathrm{sup}}^\mu \prth{f, \mathcal{T}} },
\end{equation}
where we recall that $L_{\mathrm{sup}}^\mu \prth{f, \mathcal{T}}$ is given by~\eqref{LmusupT}.
The next Proposition is a generalization to arbitrary values of $k$ and $N$ of Lemma~4.3 from~\citet{saunshi2019theoretical}, where it is assumed $k=1$ and $N = 1$.

\begin{proposition}
    \label{prop33}
    Consider the unsupervised loss $L_{\mathrm{un}}^N(f)$ from Equation~\eqref{L_un_k} with $N$ negative samples. Assume that $\rho$ is uniform over $\mathcal C$ and that $2 \leq k+1 \leq N_{\mathcal{C}}$.
    Then, any encoder function $f: \mathcal{X} \to \mathbb{R}^d$ satisfies 
    \begin{equation*}
        L_{\mathrm{sup},k}(f) \leq L_{\mathrm{sup}, k}^\mu(f) \leq \frac{k}{1-\tau_N^+} \prth{L_{\mathrm{un}}^N(f) - \tau_N^+ \log(N+1)}
    \end{equation*}
    with $\tau_N^+ = \Proba \sqprth{c_i = c, \forall i \mid (c, c_1, \cdots, c_N) \sim \rho^{\otimes N+1}}$.
\end{proposition}

The proof of Proposition~\ref{prop33} is given in the appendix.
This Proposition states that, in a setting similar to that of~\citet{saunshi2019theoretical}, on average, the $(k+1)$-way supervised classification loss is upper-bounded by the unsupervised loss (both with $N=1$ negative or $N > 1$ negatives), 
that contrastive learning algorithms actually minimize.
Therefore, these results give hints for the performances of the learned representation for downstream tasks.

Also, while~\citet{saunshi2019theoretical} only considers an unsupervised loss with $N = k$ negatives along with $(k+1)$-way tasks for evaluation, the quantities $N$ and $k$ are decoupled in Proposition~\ref{prop33}.
Furthermore, whenever $\rho$ is uniform, one has $\tau_N^+ = \sum_{c \in \mathcal{C}} \rho(c)^{N+1} = N_{\mathcal{C}}^{-N}$, which decreases to $0$ as $N \rightarrow +\infty$, so that 
a larger number of negatives $N$ makes $k/{(1-\tau_N^+)}$ smaller and closer to~$k$.
This provides a step towards a better understanding of what is actually done in practice with unsupervised contrastive learning. For instance, $N=65536$ negatives are used in~\citet{he2020momentum}.

While we considered a generic encoder $f$ and a generic setting in this Section, the next Section~\ref{subsec:conv} considers a more realistic setting of an unsupervised objective with a fixed available dataset, and the study of an \emph{explicit} algorithm for the training of $f$.

\section{Convergence of gradient descent for contrastive unsupervised learning}
\label{subsec:conv}

This section leverages results from~\citet{allen2019convergence} to provide convergence guarantees for gradient-descent based minimization of the contrastive training error, where the unsupervisedly trained encoder is an overparametrized deep neural network.

\paragraph{Deep neural network encoder.}

We consider a family of encoders $f$ defined as a deep feed-forward neural network following~\citet{allen2019convergence}. 
We quickly restate its structure here for the sake of completeness.
A deep neural encoder~$f$ is parametrized by matrices $A \in \R^{m \times d_x}$, $B \in \R^{d\times m}$ and $W_1, \dots , W_L\in \R^{m \times m}$ for some depth $L$. 
For an input $x \in \R^{d_x}$, the feed-forward output $y \in \R^d$ is given by
\begin{align*}
    g_{0} &= A x, \quad h_{0} = \phi(g_{0}), \quad g_{l} = W_l h_{l-1}, \quad  h_{l} = \phi(g_{l}) \quad \text{for} \quad l=1, \ldots, L, \\
    y &= B h_{L},
\end{align*}
where $\phi$ is the ReLU activation function. 
Note that the architecture can also include residual connections and convolutions, as explained in~\citet{allen2019convergence}.

We know from~\citet{allen2019convergence} that, provided a $\delta$-separation condition on the dataset $(x_i, y_i)$ for $i=1, \ldots, n$ with $\delta >0$ and sufficient overparametrization of the model ($m= \Omega \prth{\poly(n,L,\delta^{-1})\cdot d}$), the optimisation of the least-squares error $\frac 12 \sum_{i=1}^n \norm{\widehat y_i - y_i}_2^2$
using gradient descent provably converges to an arbitrarily low value $\epsilon > 0$, where $\widehat y_i = f(x_i)$ are the network outputs. 
Moreover, the convergence is linear i.e. the number of required epochs is $T = O(\log(1 / \eps))$, although involving a constant of order $\poly(n,L,\delta^{-1})$.
Although this result does not directly apply to contrastive unsupervised learning, we explain below how it can be adapted provided a few additional assumptions.

Ideally, we would like to prove a convergence result on the unsupervised objective defined in Equation~\eqref{L_un_k}. 
However, we need to define an objective through an explicitly given dataset so that it falls within the scope of~\citet{allen2019convergence}.
Regarding this issue, we assume in what follows that we dispose of a set of fixed triplets $(x, x^+, x^-) \in (\R^{d_x})^3$ we train on.

\paragraph{Objective function.}

Let us denote this fixed training set $\{ (x_i, x_i^+, x_i^-) \}_{i=1}^n$. 
Each element leads to an output $z_i = (f(x_i), f(x_i^+), f(x_i^-))$ by the encoder and we optimize the empirical objective
\begin{equation}
    \label{objective}
    \widehat{L}_{\mathrm{un}}(f) = \sum_{i=1}^n \softplus \big(f(x_i)^T (f(x_i^-) - f(x_i^+)) \big) = \sum_{i=1}^n \ell(z_i),
\end{equation}
where we introduced the loss function $\ell(z_i) = \ell(z_{i,1}, z_{i,2}, z_{i,3}) = \softplus (z_{i,1}^T (z_{i,3} - z_{i,2}))$ with $\zeta(x) = \log\prth{1 + e^x}$. 
Note that $\widehat{L}_{\mathrm{un}}(f) / n$ is the 
empirical counterpart of the unsupervised loss~\eqref{L_un_k}.
Our management of the set of training triplets 
can be compared to that of~\citet{wen2020convergence} who similarly fixes them in advance but uses multiple negatives and the same $x_i$ as a positive. 
However, two distinct encoders are trained therein, one for the reference sample $x_i$ and another for the rest. 
We consider here the more realistic case where a single encoder is trained.
Our approach also applies to multiple negatives, but we only use a single one here for simplicity.
We need the following data separation assumption from~\cite{allen2019convergence}.
\begin{assumption}
\label{assume:sep}
    We assume that all the samples $x \in \mathcal{X}_{\mathrm{data}} = \bigcup_{i=1}^n \{ x_i,x_i^+,x_i^- \}$ are normalized $\norm{x} = 1$
    and that there exists $\delta > 0$ such that $\norm{x - x'}_2 \geq \delta$ for any $x,x' \in \mathcal{X}_{\mathrm{data}}$.
\end{assumption}
Note that sampling the positives and negatives $x_i^+, x_i^-$ need not to be made through simple draws from the dataset. 
A common practice in contrastive learning~\citep{chen2020simple} is to use data augmentations, where we replace $x_i^\pm$ by $\psi(x_i^\pm)$ for an augmentation function $\psi$ also drawn at random. 
Such an augmentation can include, whenever inputs are color images, Gaussian noise, cropping, resizing, color distortion, rotation or a combination thereof, with parameters sampled at random in prescribed intervals.
The setting considered here allows the case where $x_i^\pm$ are actually augmentations (we won't write $\psi(x_i^\pm)$ but simply $x_i^\pm$ to simplify notations), provided that Assumption~\ref{assume:sep} is satisfied and that such augmentations are performed and fixed before training. Note that, in practice, the augmentations are themselves randomly sampled at each training iteration~\citep{chen2020simple}. 
Unfortunately, this would make the objective intractable and the convergence result we are about to derive does not apply in that case.

In order to apply the convergence result from~\citet{allen2019convergence}, we need to prove that the following gradient-Lipschitz condition
\begin{equation}
    \label{lipsmooth}
    \ell(z + z') \leq \ell(z) + \scprod{\nabla\ell(z)}{z'} + \frac{L_{\mathrm{smooth}}}{2}\norm{z'}^2
\end{equation}
holds for any $z, z' \in \R^{3d}$, for some constant $L_{\mathrm{smooth}} > 0$, where $\ell$ is the loss given by~\eqref{objective}.
However, as defined previously, $\ell$ does not satisfy~\eqref{lipsmooth} without extra assumptions.
We propose to bypass this problem by making the following additional assumption on the norms of the outputs of the encoder.
\begin{assumption}
    \label{assume}
    For each element $x\in \mathcal{X}_{\mathrm{data}}$ \textup, the output $z = f(x) \in \R^d$ satisfies
    \begin{equation*}
      \eta < \norm{z} < C  
    \end{equation*}
    during and at the end of the training of the encoder $f$\textup, for some constants $0< \eta < C < +\infty$.
\end{assumption}
In Section~\ref{sec:experiments}, we check experimentally on three datasets (see Figure~\ref{fig:norms} herein) that this assumption is rather realistic.
The lower bound $\eta > 0$ is necessary and used in Lemma~\ref{lem:super} below, while the upper bound $C$ is used in the next Lemma~\ref{lem:unsupervised-is-smooth}, which establishes the gradient-Lipschitz smoothness of the unsupervised loss $\ell$ and provides an estimation of $L_{\mathrm{smooth}}$.
\begin{lemma}
    \label{lem:unsupervised-is-smooth}
    Consider the unsupervised loss $\ell$ given by~\eqref{objective}, grant Assumption~\ref{assume} and define the set
    \begin{equation*}
        B^3 = \Big\{ z = (z_1, z_2, z_3) \in (\R^{d})^3 \; : \;  \max_{j=1,2,3} \norm{z_j}_2^2 \leq C^2 \Big\}
    \end{equation*}
    where $C > 0$ is defined in Assumption~\ref{assume}.
    Then, the restriction of $\ell$ to $B^3$ satisfies~\eqref{lipsmooth} with a constant $L_{\mathrm{smooth}} \leq 2 + 8 C^2 $.
\end{lemma}
The proof of Lemma~\ref{lem:unsupervised-is-smooth} is given in the appendix. 
Now, we can state the main result of this Section.
\begin{theorem}
    \label{theo1}
    Grant both Assumptions~\ref{assume:sep} and~\ref{assume}, let $\epsilon > 0$ and let $\widehat{L}_{\mathrm{un}}(f)$ be the loss given by~\eqref{objective}.
    Then, assuming that
    \begin{equation*}
        m \geq \Omega \Big( \frac{\poly(n,L,\delta^{-1}) \cdot d}{\epsilon} \Big), 
    \end{equation*}
    the gradient descent algorithm with a learning rate $\nu$ and a number of steps $T$ such that 
    \begin{equation*}
        \nu = \Theta \Big( \frac{d\delta}{\poly(n,L) \cdot m} \Big) \quad \text{and} \quad T = O \Big( \frac{\poly(n,L)}{\delta^2 \epsilon^2} \Big),
    \end{equation*}
    finds a parametrization of the encoder $f$ satisfying
    \begin{equation*}
        \widehat{L}_{\mathrm{un}}(f) \leq \epsilon.
    \end{equation*}
\end{theorem}
The proof of Theorem~\ref{theo1} is given in the appendix.
Although it uses Theorem~6 from~\citet{allen2019convergence}, it is actually \emph{not} an immediate consequence of it.
Indeed, in our case, the Theorem~6 therein only allows us to conclude that $\| \nabla\widehat{L}_{\mathrm{un}}(f) \| \leq \epsilon$, where the gradient is taken w.r.t. the outputs of $f$. 
The convergence of the objective itself is obtained thanks to the following Lemma whose proof is given in the appendix.
\begin{lemma}
    \label{lem:super}
    Grant Assumption~\ref{assume} and assume that the parameters of the encoder $f$ are optimized so that $\| \nabla \widehat{L}_{\mathrm{un}}(f) \| \leq \epsilon$ with $\epsilon < \eta/2$, where $\eta$ is defined in Assumption~\ref{assume}. Then, for any $i=1, \ldots, n$, we have $\ell(z_i) \leq 2\epsilon/\eta$ where $z_i = (f(x_i), f(x_i^+), f(x_i^-))$.
\end{lemma}
This Lemma is crucial for proving Theorem~\ref{theo1} as it allows to show, in this setting, that the reached critical point is in fact a global minimum.

A natural idea would be then to combine Theorem~\ref{theo1} with Proposition~\ref{prop33} in order to prove that gradient descent training of the encoder using the unsupervised contrastive loss helps to minimize the supervised loss.
This paper makes a step towards such a result, but let us stress that it requires much more work, to be considered in future papers, the technical problems to be addressed being as follows.
Firstly, the result of Theorem~\ref{theo1} applies to $\widehat{L}_{\mathrm{un}}(f)$ and cannot be directly extrapolated on $L_{\mathrm{un}}(f)$.
Doing so would require a sharp control of the generalization error, 
while Theorem~\ref{theo1} is about the training error only.
Secondly, Assumption~\ref{assume:sep} requires that all samples are separated and, in particular, distinct.
This cannot hold when the objective is defined through an expectation as we did in Section~\ref{sec:sec3}. 
Indeed, it would be invalidated simply by reusing a sample in two different triples.

\section{Experiments} 
\label{sec:experiments}

In this section, we report experiments that illustrate our theoretical findings.

\paragraph{Datasets and Experiments.}

We use a small convolutional network as encoder on MNIST ~\citep{lecun-mnisthandwrittendigit-2010} and FashionMNIST~\citep{xiao2017/online}, and VGG-16~\citep{Simonyan15vgg} on CIFAR-10~\citep{krizhevsky2009learning}.
Experiments are performed with PyTorch~\citep{NEURIPS2019_9015pytorch}. 

\paragraph{Results.} 

Figure~\ref{fig:lem31} provides an illustration of Lemma~\ref{lem81},
where we display the values of $L_{\mathrm{un}}$ (i.e., $L^N_{\mathrm{un}}$ with $N=1$) and $L^{\mu}_{\mathrm{sup}}(f, \mathcal{C})$ along training iterations over $5$ separate runs (and their average).
We observe that Inequality~\eqref{ineq31} is satisfied on these experiments, even when the $\log N_{\mathcal{C}}$  term is discarded. 
Moreover, both losses follow a similar trend. 
Figure~\ref{fig:lem32} illustrates Lemma~\ref{lem32} for several values of $N$. 
Once again, we observe that both losses behave similarly, and that  Inequality \eqref{ineq32} seems to hold even without the $1/p^{\rho}_{cc}$ term (removed for these displays).

\begin{figure}[h]
    \centering
    \includegraphics[width=\textwidth]{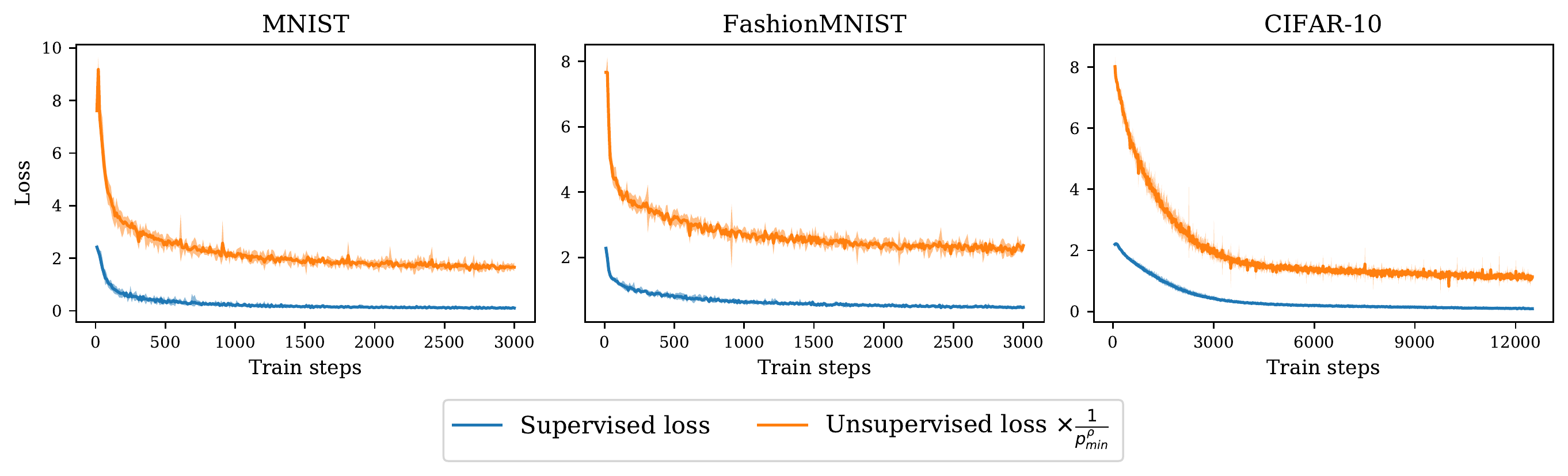}
    \caption{Illustration of Lemma~\ref{lem81}: we observe that Inequality~\eqref{ineq31} is satisfied on these examples, even without the $\log N_{\mathcal{C}}$ term, and that both losses behave similarly (5 runs are displayed together with their average).}
    \label{fig:lem31}
\end{figure}

\begin{figure}[h!]
    \centering
    \includegraphics[width=\textwidth]{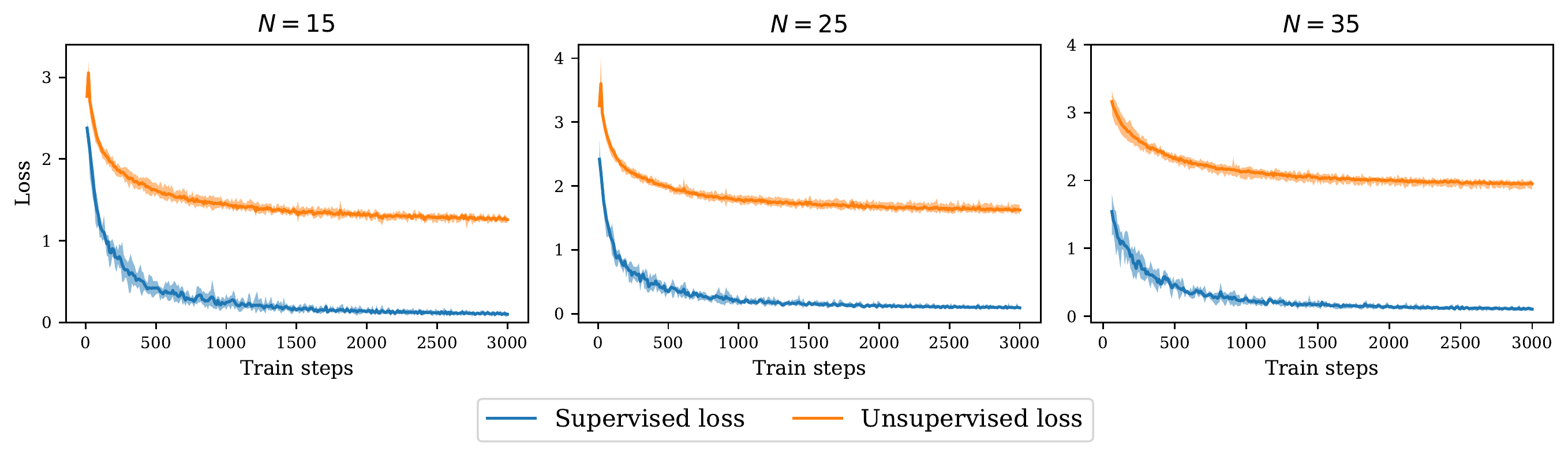} 
    \caption{Illustration of Lemma~\ref{lem32} with $N=15,25,35$ on MNIST. We observe again that both the unsupervised and supervised losses behave similarly and that Inequality \eqref{ineq32} is satisfied in these experiments, even without the $1/p^{\rho}_{cc}$ factor (5 runs are displayed together with their average).}
    \label{fig:lem32}
\end{figure}

Finally, Figure~\ref{fig:norms} displays the minimum and maximum Euclidean norms of the outputs of the encoder along training. 
On these examples, we observe that one can indeed assume these norms to be lower and upper bounded by constants, as stated in Assumption~\ref{assume}.

\begin{figure}[h]
    \centering
    \includegraphics[width=\textwidth]{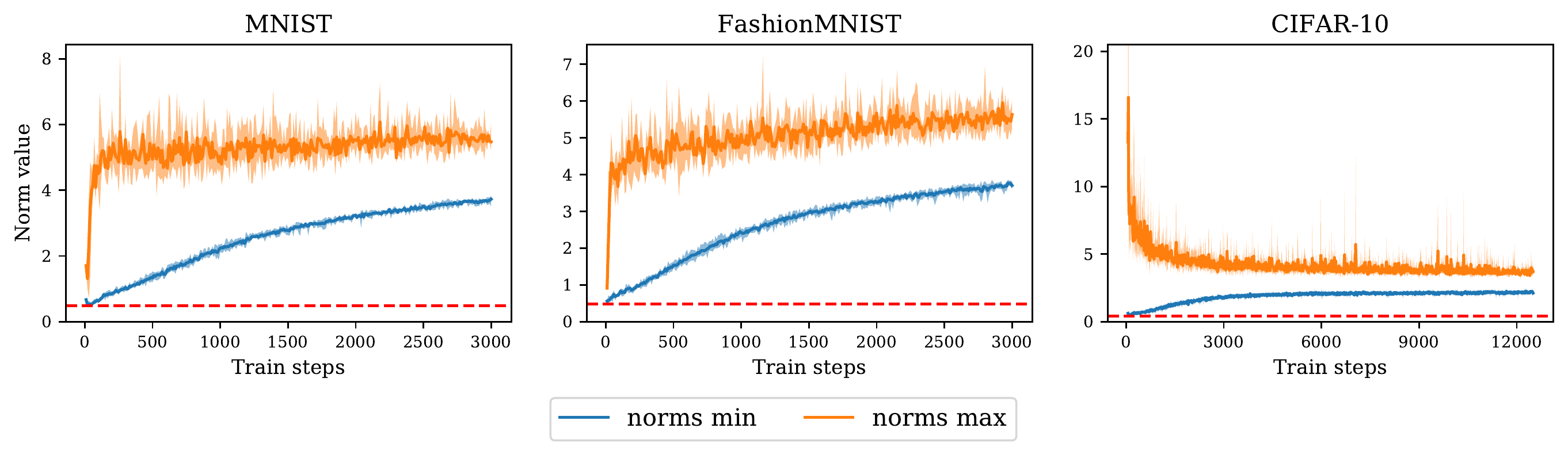}
    \caption{Minimum and maximum Euclidean norms of the outputs of the encoder along  contrastive unsupervised training. We observe that Assumption~\ref{assume} is satisfied on these examples (5 runs are displayed together with their average), the dashed line shows that the minimum norms are away from $0$ even in the early iterations.}
    \label{fig:norms}
\end{figure}

\section{Conclusion}

This work provides extensions to previous results on contrastive unsupervised learning, in order to somewhat improve the theoretical understanding of the performance that is empirically observed with pre-trained encoders used for subsequent supervised task.
The main hindrance to tighter bounds in Section~\ref{sec:sec3} is the blind randomness of negative sampling, which is unavoidable in the unsupervised setting. 
Section~\ref{subsec:conv} explains how recent theoretical results about gradient descent training of overparametrized deep neural networks can be used for unsupervised contrastive learning,
and concludes with an explanation of why combining the results from Sections~\ref{sec:sec3} and~\ref{subsec:conv} requires many extra technicalities to be considered in future works.  
Let us conclude by stressing, once again, our motivations for doing this: unsupervised learning theory is much less developed than supervised learning theory, and recent empirical results (see  Section~\ref{sec:introduction}) indicate that some forms of contrastive learning enable the learning of powerful representations without supervision. 
In many fields of application, labels are too difficult, too expensive or too invasive to obtain (in medical applications, see for instance~\citet{ching2018opportunities}).
We believe that a better understanding of unsupervised learning is therefore of utmost importance. 

\paragraph{Acknowledgements.} 

This work was funded in part by the French government under management of Agence Nationale de la Recherche as part of the ``Investissements d’avenir'' program, reference ANR-19-P3IA-0001 (PRAIRIE 3IA Institute).
Yiyang Yu was supported by grants from Région Ile-de-France.

\section{Proofs for Section~\ref{sec:sec3}}

Apart from the similarity between the unsupervised and supervised loss, the proof of Lemma~\ref{lem81} uses properties of log-sum-exp.

\begin{proof}[Proof of Lemma~\ref{lem81}]
We first rewrite the unsupervised loss as:
\[L_{\mathrm{un}}(f) = \E_{(x, x^+) \sim \mathcal{D}_{\mathrm{sim}} , x^- \sim \mathcal{D}_{\mathrm{neg}}} \log \prth{1 + \exp \prth{ f(x)^T (f(x^-) - f(x^+)) }}  \]
where we recognize the $\softplus$ function $\softplus(x) = \log(1 + e^x)$. We start by using Jensen's inequality
\begin{align*}
L_{\mathrm{un}}(f) &= \E_{\substack{(x, x^+) \sim \mathcal{D}_{\mathrm{sim}} \\ x^- \sim \mathcal{D}_{\mathrm{neg}}}} \sqprth{\softplus  \prth{ f(x)^T (f(x^-) - f(x^+)) }} \\
&\geq \E_{c,c^- \sim \rho, x \sim \mathcal{D}_c} \sqprth{\softplus \prth{ f(x)^T (\mu_{c^-} - \mu_c) }} \\
&\geq p^\rho_{\mathrm{min}} \E_{c \sim \rho, x \sim \mathcal{D}_c} \sqprth{\max_{c^-} \softplus \prth{ f(x)^T (\mu_{c^-} - \mu_c) }} \\
&= p^\rho_{\mathrm{min}} \E_{c \sim \rho, x \sim \mathcal{D}_c} \sqprth{\max_{c^-} \lse \prth{0,  f(x)^T (\mu_{c^-} - \mu_c) }} \\
&\geq p^\rho_{\mathrm{min}} \prth{\E_{c \sim \rho, x \sim \mathcal{D}_c} \sqprth{\lse \prth{ f(x)^T (\mu_{c_1} - \mu_c) , \dots,  f(x)^T (\mu_{c_{N_{\mathcal{C}}}} - \mu_c) }} - \log N_{\mathcal{C}}}\\
&= p^\rho_{\mathrm{min}} \prth{L_{\mathrm{sup}}^{\mu}(f, \mathcal{C}) - \log N_{\mathcal{C}}}
\end{align*}
where we have used properties of the log-sum-exp function
\[\max(x_1, \dots, x_n) \leq \lse (x_1, \dots, x_n) \leq \max (x_1, \dots, x_n) + \log n,\]
the fact that $\lse$ is non-negative whenever one of its arguments is, and for $x \in \R^{2n}$ we have
\[\lse (x) = \lse ( \lse(x_1, x_2), \dots, \lse(x_{2n-1}, x_{2n})) \leq \max_{j = 1, \dots, n}\lse(x_{2j-1}, x_{2j}) + \log n.\]
\end{proof}

The proof of Lemma~\ref{lem32} considers the sample draws where all classes are represented.

\begin{proof}[Proof of Lemma~\ref{lem32}]
Let $I \in [N_{\mathcal{C}}]^N$ the random vector of classes for each negative sample ($I \sim \rho^{\otimes N}$) and let $J$ be the set of represented classes i.e. $J = \braces{I_j \:\: \middle| \:\: j \in [N]}$. We have, again with Jensen's inequality
\begin{align*}
L^N_{\mathrm{un}}(f) &= \E_{x,x^+,x_1^-, \dots,x_N^-}\sqprth{ \lse \prth{0, f(x)^T (f(x_1^-) - f(x^+)) , \dots , f(x)^T (f(x_N^-) - f(x^+)) }} \\
&\geq \E_{c\sim \rho, I\sim \rho^{\otimes N}, x \sim \mathcal{D}_c} \sqprth{ \lse \prth{0, f(x)^T (\mu_{I_1} - \mu_c) , \dots , f(x)^T (\mu_{I_N} - \mu_c) }} \\
&\geq \Proba \prth{ \abs{J} = N_{\mathcal{C}}} \E_{\substack{c\sim \rho \\ I\sim \rho^{\otimes N} \\ x \sim \mathcal{D}_c}}\sqprth{ \lse \prth{0, f(x)^T (\mu_{I_1} - \mu_c) , \dots , f(x)^T (\mu_{I_N} - \mu_c)} \mid \abs{J} = N_{\mathcal{C}}} \\
&\geq \Proba \prth{ \abs{J} = N_{\mathcal{C}}} L^\mu_{\mathrm{sup}}(f, \mathcal{C}),
\end{align*}
where we used that for $\mathcal{S} \subset [n]$ and $x \in\R^n$ we have $\lse(x_{\mathcal{S}}) \leq \lse(x)$ with $x_{\mathcal{S}}$ the restriction of $x$ to the indices in $\mathcal{S}$. Finally, we have $\Proba \prth{ \abs{J} = N_{\mathcal{C}}} = p^\rho_{cc}(N)$.
\end{proof}

We restate Proposition~\ref{prop33} for cases $N=1$ and $N>1$. The proof uses Jensen's inequality and the uniformity of $\rho$.
\begin{proposition}[3.3 (restated)]
Consider the unsupervised loss $L_{\mathrm{un}}^N(f)$ from Equation~\eqref{L_un_k} with $N$ negative samples. Assume that $\rho$ is uniform over $\mathcal C$ and that $2 \leq k+1 \leq N_{\mathcal{C}}$.
Then, 
\begin{enumerate}
    \item[(1)] any encoder function $f: \mathcal{X} \to \mathbb{R}^d$ satisfies 
    \begin{equation*}
    L_{\mathrm{sup},k}(f) \leq L_{\mathrm{sup}, k}^\mu(f) \leq \frac{k}{1-\tau^+} \prth{L_{\mathrm{un}}(f) - \tau^+}     
    \end{equation*}
    with $\tau^+ = \Proba_{c, c' \sim \rho^2}\prth{c=c'}$\textup, where $L_{\mathrm{un}}(f)$ is the unsupervised loss from Equation~\eqref{L_un_k} with $N = 1$ negative sample;
    \item[(2)] more generally, 
    \begin{equation*}
        L_{\mathrm{sup},k}(f) \leq L_{\mathrm{sup}, k}^\mu(f) \leq \frac{k}{1-\tau_N^+} \prth{L_{\mathrm{un}}^N(f) - \tau_N^+ \log(N+1)}
    \end{equation*}
    with $\tau_N^+=\Proba(c_i = c, \forall i \mid c \sim \rho, (c_1, \cdots, c_N) \sim \rho^N)$, and where $L_{\mathrm{un}}^N(f)$ is the unsupervised loss from Equation~\eqref{L_un_k}.
\end{enumerate}
\end{proposition}

\begin{proof}[Proof of Proposition~\ref{prop33}]
Let's start with (1). By Jensen's inequality, then use $\log = \log_2$ without loss of generality, 
and split the expectation into cases $c^- \neq c$ and $c^- = c$,
\begin{align*}
    L_{\mathrm{un}}(f) &= \E_{(c, c^-) \sim \rho^2 } \E_{x,x^+ \sim \mathcal{D}_c, x^- \sim \mathcal{D}_{c^-}} \sqprth{ \log \prth{  1 + \exp \prth{f(x)^T \prth{f(x^-)-f(x^+)}}}} \\
    &\geq \E_{(c, c^-) \sim \rho^2 , x \sim \mathcal{D}_{c}} \sqprth{ \log \prth{ 1 + \exp \prth{f(x)^T \prth{\mu_{c^-} - \mu_c}}}} \\
    & = (1-\tau^+) \E_{c \sim \rho, x \sim \mathcal{D}_{c}} \E_{c^- \sim \rho} \sqprth{ \log \prth{  1 + \exp \prth{f(x)^T \prth{\mu_{c^-} - \mu_c}}} \middle | c^- \neq c} + \tau^+.
\end{align*}

Let us write explicitly the uniform distribution $\rho$ on $\mathcal{C}$. On the one hand,
\begin{align*}
&\E_{c^- \sim \rho} \sqprth{ \log \prth{ 1 + \exp \prth{f(x)^T \prth{\mu_{c^-} - \mu_c}}} \middle | c^- \neq c } \\ 
=& \sum_{c^- \in \mathcal{C} \backslash \{c\}} \frac{1}{N_{\mathcal{C}} -1} \log \prth{ 1 + \exp \prth{f(x)^T \prth{\mu_{c^-} - \mu_c}}},
\end{align*}
on the other hand,
And this is for every $c^- \in \mathcal{C} \backslash \{c\}$. We rearrange the double sum according to $c^-$

Hence, using the uniformity of $\rho$,
\begin{align*}
&\E_{c^- \sim \rho} \sqprth{ \log \prth{ 1 + \exp \prth{f(x)^T \prth{\mu_{c^-} - \mu_c}}} \middle | c^- \neq c }\\
=& \frac{1}{k} \E_{{c_1,\dots,c_k} \sim \rho^{\otimes k}} \sqprth{ \sum_{i=1}^k \log \prth{  1 + \exp \prth{f(x)^T \prth{\mu_{c_i} - \mu_c}}} \middle | \braces{c, c_1,\dots,c_k} \text{distinct} }\\
\geq& \frac{1}{k} \E_{{c_1,\dots,c_k} \sim \rho^{\otimes k}} \sqprth{ \log \prth{1 + \sum_{i=1}^k \exp \prth{f(x)^T \prth{\mu_{c_i} - \mu_c}}} \middle | \braces{c, c_1,\dots,c_k} \text{distinct}} .
\end{align*}

That means we have
\begin{align*}
L_{\mathrm{un}}(f) &\geq \frac{1-\tau^+}{k} \E_{\substack{c \sim \rho, x \sim \mathcal{D}_{c} \\ {c_1,\dots,c_k} \sim \rho^{\otimes k}}} \sqprth{ \log \prth{1 + \sum_{i=1}^k \exp \prth{f(x)^T \prth{\mu_{c_i} - \mu_c}}} \middle | \braces{c, c_1,\dots,c_k} \text{distinct} } + \tau^+ \\
&= \frac{1-\tau^+}{k} \E_{\mathcal{T} \sim \mathcal{D}^{k+1}} \E_{(x,c) \sim \mathcal{D}_{\mathcal{T}}} \sqprth{ -\log \prth{ \frac{\exp(f(x)^T \mu_c)}{\exp(f(x)^T \mu_c) + \sum_{\substack{c^- \in \mathcal{T}\\ c^- \neq c}} \exp \prth{f(x)^T \mu_{c^-}}}}} + \tau^+ \\
&= \frac{1-\tau^+}{k} L_{\mathrm{sup}, k}^\mu(f) + \tau^+.
\end{align*}

As for (2), again by Jensen's inequality, and split the expectation into cases $c_i^- = c, \forall i$ and $\exists c_i^- \neq c$, 
\begin{align*}
    L_{\mathrm{un}}^N(f) &= \E_{(c, c_i^-) \sim \rho^{N+1}} \E_{x,x^+ \sim \mathcal{D}_c, x_i^- \sim \mathcal{D}_{c_i^-}} \sqprth{ \log \prth{ 1 + \sum_{i=1}^N \exp \prth{f(x)^T \prth{f(x_i^-)-f(x^+)}}}} \\
    &\geq \E_{(c, c_i^-) \sim \rho^{N+1} , x \sim \mathcal{D}_{c}} \sqprth{ \log \prth{ 1 + \sum_{i=1}^N \exp \prth{f(x)^T \prth{\mu_{c_i^-} - \mu_c}}}} \\
    & = (1-\tau_N^+) \E_{\substack{c \sim \rho\\ x \sim \mathcal{D}_{c}}} \E_{c_i^- \sim \rho^N} \sqprth{\log \prth{ 1 + \sum_{i=1}^N \exp \prth{f(x)^T \prth{\mu_{c^-} - \mu_c}}} \middle|  \exists c_i^- \neq c} + \tau_N^+ \log(N+1)
\end{align*}
with $$\tau_N^+=\Proba(c_i = c, \forall i \mid c \sim \rho, c_i \sim \rho^N) = \sum_{c \in \mathcal{C}} \rho(c)^{N+1} = N_{\mathcal{C}}^{-N}.$$
Considering the fact that
\begin{align*}
\E_{c_i^- \sim \rho^N} &\sqprth{\log \prth{ 1 + \sum_{i=1}^N \exp \prth{f(x)^T \prth{\mu_{c^-} - \mu_c}}} \middle|  \exists c_i^- \neq c} \geq\\ &\E_{c^- \sim \rho} \sqprth{ \log \prth{ 1 + \exp \prth{f(x)^T \prth{\mu_{c^-} - \mu_c}}} \middle | c^- \neq c }, 
\end{align*}
then by similar computations as in (1), we have
\[ L_{\mathrm{un}}^N(f) \geq \frac{1-\tau_N^+}{k} L_{\mathrm{sup}, k}^\mu(f) + \tau_N^+ \log(N+1). \]
\end{proof}

\section{Proofs for Section~\ref{subsec:conv}}

Let us first prove that under Assumption~\ref{assume}, the objective is gradient-Lipschitz w.r.t. the network outputs.

\begin{lemma}[Lemma~4.1]
	\label{lem34rest}
Consider the unsupervised loss $\ell$ given by~\eqref{objective}, grant Assumption~\ref{assume} and define the set
    \begin{equation*}
        B^3 = \Big\{ z = (z_1, z_2, z_3) \in (\R^{d})^3 \; : \;  \max_{j=1,2,3} \norm{z_j}_2^2 \leq C^2 \Big\}
    \end{equation*}
where $C > 0$ is defined in Assumption~\ref{assume}.
Then, the restriction of $\ell$ to $B^3$ satisfies~\eqref{lipsmooth} with a constant $L_{\mathrm{smooth}} \leq 2 + 8 C^2 $.	
\end{lemma}

\begin{proof}
We will prove this result by bounding the norm of the Hessian matrix.

Let us write the gradient of $\ell(z)$ with respect to $z$ first. We have $z \in \R^{3d}$. For ease of writing, we define the matrices $A_1, A_2, A_3 \in \R^{3d \times d}$ as 
\[A_1 = \begin{pmatrix} I_d \\ 0_d \\ 0_d\end{pmatrix} \quad A_2 = \begin{pmatrix} 0_d \\ I_d \\ 0_d\end{pmatrix} \quad A_3 = \begin{pmatrix} 0_d \\ 0_d \\ I_d\end{pmatrix} \quad \]
where $I_d, 0_d \in \R^{d\times d}$ are the identity and zero matrix respectively. With this notation, we have $z_i = A_i^T z$ for $i=1,2,3$ the three contiguous thirds of $z$'s coordinates.

Our purpose is to compute
\[\frac{\partial}{\partial z} \ell(z) = \frac{\partial}{\partial z}\sqprth{- \log \prth{\frac{\exp\prth{z_1^T z_2}}{\exp\prth{z_1^T z_2} + \exp\prth{z_1^T z_3} }}}. \]

Denote $\cos_{i,j} = z_i^T z_j$, we can now compute for $i,j \in \braces{1,2,3}$ ($i \neq j$)
\[\frac{\partial }{\partial z} \cos_{i,j} = \prth{A_i A_j^T + A_j A_i^T } z  =: \partial \cos_{i,j} \in \R^{3d}.\]

Now, denote $v = \softmax\prth{\cos_{1,2}, \cos_{1,3}} \in \R^2$, we can write
\[\frac{\partial}{\partial z} \ell (z) = (v_1 - 1)\partial \cos_{1,2} + v_2 \partial \cos_{1,3}.\]

We proceed with the following computation
\[\frac{\partial^2}{\partial z^2} \cos_{i,j} = A_i A_j^T + A_j A_i^T, \]
which we will denote simply as $\partial^2 \cos_{i,j}$. Before we get the Hessian of loss, we still need to compute
\[\partial v:= \frac{\partial v}{\partial z} = (\diag(v) - vv^T)\begin{pmatrix} \partial \cos_{1,2}^T \\ \partial \cos_{1,3}^T \end{pmatrix} \in \R^{2\times 3d}. \]

Now we can write
\[\frac{\partial^2}{\partial z^2} \ell(z) = (v_1 - 1)\partial^2 \cos_{1,2} + v_2 \partial^2 \cos_{1,3} + \begin{pmatrix}\partial \cos_{1,2} & \partial \cos_{1,3}\end{pmatrix} \partial v.\]

We can now estimate the norm of this matrix which will provide an estimation for the Lipschitz constant.

We find that
\[\norm{\partial\cos_{i,j}} \leq 2\max \prth{\norm{z_i}, \norm{z_j}},\]
keeping in mind that the matrix $\diag(v) - vv^T$ has norm at most $1/2$, this leads to
\[\norm{\begin{pmatrix}\partial \cos_{1,2} & \partial \cos_{1,3}\end{pmatrix} \partial v} = 8\max_{i,j} \prth{\norm{z_i}\norm{z_j}}.\]

We have also that $\norm{\partial^2 \cos_{i,j}} = 1$.

All in all, we have $\norm{\frac{\partial^2}{\partial z^2} \ell(z)} = 2 + 8 \max_{i,j} \prth{\norm{z_i}\norm{z_j}}$. Recalling that we restricted $\R^{3d}$ so that we have $\max_i \norm{z_i} \leq C$ 
the result follows.
\end{proof}

Theorem~\ref{theo1} is actually obtained in two steps. First, Theorem 6 from~\citet{allen2019convergence} allows us to obtain that the gradient of the objective $\nabla \widehat{L}_{\mathrm{un}}(f)$ with respect to the network outputs reaches arbitrarily low values. Then, combining this with Assumption~\ref{assume}, this result can be extended into the objective itself.

Following appendix A of~\citet{allen2019convergence}, we need to define the $\loss$ vectors for our model. These are originally defined as $\loss_i = y_i - y_i^*$ ($y_i$ and $y_i^*$ are respectively the output and label corresponding to an input $x_i$ from the dataset) for the $\ell^2$ loss. More generally, for a network output $z_i$, they are defined as
\[\loss_i = \nabla_z \ell(z_i).\]

Following the unsupervised training protocol, samples are fed into the network three at a time $x, x^+$ and $x^-$. 
Let us denote $\theta$ the parameters of the network $f$, for a triplet $(x_i, x_i^+, x_i^-)$, the trick is to write:
\[\frac{\partial}{\partial \theta} \ell (z_i) = \frac{\partial z}{\partial \theta} \underbrace{\frac{\partial}{\partial z} \ell (z_i)}_{\loss}\]
with $z_i$ the concatenation of $f(x_i), f(x_i^+), f(x_i^-)$.

By denoting $(x_1, x_2, x_3) = (x_i, x_i^+, x_i^-)$, the previous writing is equivalent to
\[\sum_{j=1}^3 \frac{\partial f(x_j)}{\partial \theta} A_j^T \frac{\partial}{\partial z} \ell (z_i)\]
and by letting $\loss_{i,j} = A_j^T \frac{\partial}{\partial z} \ell (z_i)$, we obtain a triplet of loss vectors for each data triple (matrices $A_j$ defined in the previous proof). 

\begin{lemma}\label{lemB1}
Grant Assumption~\ref{assume:sep} and let $\widehat{L}_{\mathrm{un}}(f)$ be the loss incurred by $f$:
\[\widehat{L}_{\mathrm{un}}(f) = \sum_{i=1}^n \ell(f(x_i), f(x_i^+), f(x_i^-))\]
and let $\epsilon > 0$ be the desired precision. Then, assuming $m \geq \Omega\prth{\poly(n,L,\delta^{-1})\cdot d\epsilon^{-1}}$, the gradient descent with learning rate $\nu = \Theta \prth{\frac{d\delta}{\poly(n,L) \cdot m}}$ finds parameters such that $$\| \nabla \widehat{L}_{\mathrm{un}}(f) \| \leq \epsilon$$
after a number of steps $T = O\prth{\frac{\poly(n,L)}{\delta^2 \epsilon^2}}$.
\end{lemma}

\begin{proof}
This result follows from~\citet{allen2019convergence} (see Theorem 6 and appendix A). It corresponds to the case of a non-convex bounded loss function. We only need to check the used loss function $\ell$ is bounded and gradient-Lipschitz smooth. The latter condition is verified due to Lemma~\ref{lem34rest} and Assumption~\ref{assume}.

As for the boundedness, it is also a consequence of Assumption~\ref{assume} and the fact that the softplus function satisfies
\[\zeta(x) \sim^{x\to +\infty} x \quad \text{and} \lim_{x \to - \infty} \zeta(x) = 0.\]

\end{proof}

From here, we can derive a result for the objective itself (Theorem~\ref{theo1}) thanks to the following Lemma. 

\begin{lemma}[Lemma 4.2]
	\label{lem42rest}
Grant Assumption~\ref{assume} and assume that the parameters of the encoder $f$ are optimized so that $\| \nabla \widehat{L}_{\mathrm{un}}(f) \| \leq \epsilon$ with $\epsilon < \eta/2$, where $\eta$ is defined in Assumption~\ref{assume}. Then, for any $i=1, \ldots, n$, we have $\ell(z_i) \leq 2\epsilon/\eta$ where $z_i = (f(x_i), f(x_i^+), f(x_i^-))$.
\end{lemma}

\begin{proof}
Since we assume $\| \nabla \widehat{L}_{\mathrm{un}}(f) \| \leq \epsilon$, this also implies that $\max_{i,j }\norm{\loss_{i,j}} \leq \epsilon$ (see Theorem 3 of \cite{allen2019convergence} and its variant in appendix A).

We can write the norms $\norm{\loss_{i,j}}$ as:
\begin{align*}
    \norm{\loss_{i,1}} &= \norm{(v_1 - 1) z_{i,2} + v_2 z_{i,3}} \\
    \norm{\loss_{i,2}} &= \abs{v_1 - 1} \norm{z_{i,1}} \\
    \norm{\loss_{i,3}} &= v_2 \norm{z_{i,1}}
\end{align*}
where we defined 
$v = \softmax(z_1^T z_2, z_1^T z_3)$.

Thanks to Assumption~\ref{assume}, we can argue that $\norm{z_{i,j}} \geq \eta$. These quantities can be small for $v_1 \to 1$ and $v_2 \to 0$. Since we have $\max_{i,j }\norm{\loss_{i,j}} \leq \epsilon$, this implies in particular that for all $i$ we get $\norm{\loss_{i,3}} \leq \epsilon$ which means $v_2 \leq \epsilon /\eta $, and we have $v_2 = \sigma (z_1^T (z_3 - z_2))$. So for an instance $i \in [n]$ the loss term in the objective is:
\begin{align*}
    \softplus(z_{i,1}^T (z_{i,3} - z_{i,2})) &= \log (1 + \exp(z_{i,1}^T (z_{i,3} - z_{i,2}))) \\
    &= - \log(\sigma ( - z_{i,1}^T (z_{i,3} - z_{i,2}))) \\
    &= - \log(1 - \sigma ( z_{i,1}^T (z_{i,3} - z_{i,2}))) \\
    &= - \log(1 - v_2) \leq \frac{v_2}{1 - v_2} \leq  2v_2 \leq 2\epsilon/\eta,
\end{align*}
where we used the inequality $-\log (1 - x)  \leq \frac{x}{1-x}$ for $0 \leq x < 1$, and the assumption that $\epsilon < \eta/2$.
\end{proof}

Lemma~\ref{lem42rest} allows us to deduce that the objective is well optimized (we treated the loss term for a single triplet here but the same methods can be applied to the whole objective with a number of gradient steps which is still polynomial).

\begin{proof}[Proof of Theorem~\ref{theo1}]
Theorem~\ref{theo1} is the consequence of combining Lemma~\ref{lemB1} applied using $\frac{\epsilon \eta }{2n}$ instead of $\epsilon$ and Lemma~\ref{lem42rest} (the $1/n$ factor can be absorbed by the $\poly(n,L)$ factors in the bounds of Lemma~\ref{lemB1}).
\end{proof}


\end{document}